\documentclass[runningheads]{llncs}
\usepackage[T1]{fontenc}
\usepackage{amsmath}
\usepackage[whole]{bxcjkjatype}
\usepackage{graphicx}
\usepackage{soul}
\usepackage{mathtools}
\usepackage{stmaryrd}
\usepackage{xcolor}
\usepackage{soul}
\usepackage{bm}
\usepackage{amssymb}
\usepackage{comment}

\DeclareMathOperator*{\argmax}{arg\,max}

\newcommand{\m}[1]{[\![#1]\!]}
\newcommand{\ms}[1]{[\![#1]\!]}
\newcommand{\pms}[1]{[\![\![#1]\!]\!]}
\newcommand{\mm}[1]{(\!(#1)\!)}
\newcommand{\ams}[1]{(\!(#1)\!)}

\newcommand{\pams}[1]{(\!(\!(#1)\!)\!)}

\newcommand{\cent}{\mathrel{\scalebox{1}[1.5]{$\shortmid$}\mkern-3.1mu\raisebox{0.1ex}{$=$}}}
\newcommand{\ent}{\mathrel{\scalebox{1}[1.5]{$\shortmid$}\mkern-3.1mu\raisebox{0.1ex}{$\equiv$}}}
\begin{document}
\title{Towards Unifying Perceptual Reasoning and Logical Reasoning}
\author{Hiroyuki Kido\orcidID{0000-0002-7622-4428}}
\authorrunning{H. Kido}
\institute{Cardiff University, Park Place, Cardiff, CF10 3AT, UK\\
\email{KidoH@cardiff.ac.uk}}
\maketitle              
\begin{abstract}
An increasing number of scientific experiments support the view of perception as Bayesian inference, which is rooted in Helmholtz's view of perception as unconscious inference. Recent study of logic presents a view of logical reasoning as Bayesian inference. In this paper, we give a simple probabilistic model that is applicable to both perceptual reasoning and logical reasoning. We show that the model unifies the two essential processes common in perceptual and logical systems: on the one hand, the process by which perceptual and logical knowledge is derived from another knowledge, and on the other hand, the process by which such knowledge is derived from data. We fully characterise the model in terms of logical consequence relations.
\keywords{Generative model  \and Logic \and Neuroscience \and Learning.}
\end{abstract}
%
\section{Introduction}
Bayesian brain hypothesis \cite{knill:04}, free-energy principle \cite{friston:10} and predictive coding \cite{Hohwy:08} are key brain theories that account for action, perception and learning. They commonly argue that the brain unconsciously and actively predicts and perceives the world using the belief of states of the world. One would be surprised if one suddenly heard a loud noise. A reasonable explanation for the sense of surprise is that, using the belief of states of the world, the brain predicts that its occurrence will not happen. However, what is not yet explored or understood in this line of research is logical reasoning, the core of human rational thought. One would be surprised by the scene of a person living a hundred years ago having a smartphone in a film. A reasonable explanation for this sense of surprise is that the brain predicts that this scenario will not happen on the basis of logical reasoning from what one believes and observes. An important open question is what theory of reasoning gives a unified account for these kinds of perceptual reasoning and logical reasoning.
\par
This paper gives a simple probabilistic model to tackle the question. The current computational approaches to perception (e.g., \cite{Smith:22,Hohwy:08,Hohwy:14,knill:04,Funamizu:15}) define how perceptual knowledge is derived from another knowledge, including observed data, using Bayes' theorem or its approximation. In contrast, our model distinctively defines how such knowledge is derived from stored data. The model unifies two essential processes not only of the perceptual systems but of many current logical systems: on the one hand, the process by which the posterior is derived from the likelihood and prior, and on the other hand, the process by which the likelihood and prior are derived from data. The complexity of reasoning with our model is shown to be linear with respect to the number of data. 
\par
The contributions of this paper are summarised as follows. We present a theory of reasoning that underlies perceptual reasoning and logical reasoning. In particular, this paper gives researchers in neuroscience another type of computational model of perception that provides another view of perception as Bayesian inference. This paper also gives researchers in logic an algorithm to rationally reason from data. Future work includes the neuroscientific validity of the model.
\par
This paper is organised as follows. We motivate our work in Section 2 with simple examples of perceptual reasoning and logical reasoning. We define a simple probabilistic model in Section 3, and then look at the properties of the model in Section 4 for its correctness. We conclude with discussion in Section 5.
\section{Motivating Examples}\label{sec:motivation}
\begin{figure}[t]
\begin{center}
\includegraphics[scale=0.35]{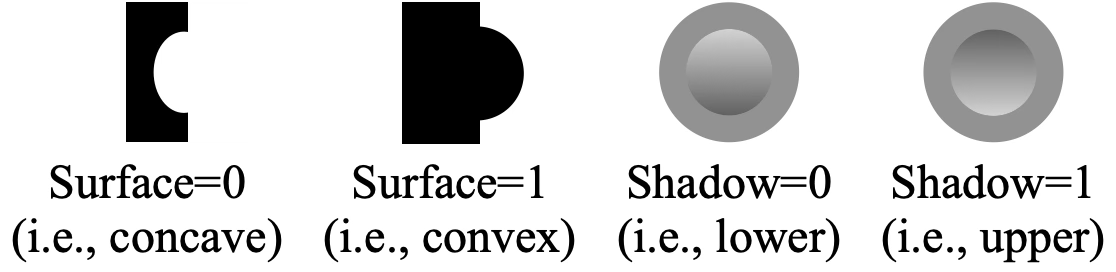}
\caption{Most people tend to perceive a concave surface when a shadow appears in the upper portion of the disk. This is explained as the result of our unconscious belief that light sources typically come from above.}
\label{fig:illusion}
\end{center}
\end{figure}
Human perception is often described using Bayesian inference. Consider which surface (concave or convex) is perceived when a shadow appears in the upper portion of the disk (see Figure \ref{fig:illusion})? To model this phenomenon, the authors \cite{Smith:22} define the prior distribution, denoted by $p(\text{Surface})$, and the likelihood function, denoted by $p(\text{Surface}|\text{Shadow})$, as follows.
\begin{eqnarray*}
&&p(\text{Surface}=\text{concave})=0.45\\
&&p(\text{Surface}=\text{convex})=0.55\\
&&p(\text{Shadow}=\text{upper}|\text{Surface}=\text{concave}, \text{light above})=0.9\\
&&p(\text{Shadow}=\text{upper}|\text{Surface}=\text{convex}, \text{light above})=0.1
\end{eqnarray*}
The prior represents the belief about the surface of the object before observing the shadow. The above equations show that the probability of convex is slightly higher than the probability of concave. The likelihood is the belief that the surface causes the upper shadow. Here, `light above' is the unconscious belief that light sources typically come from above. The above equations show that the probability that the concave surface causes the upper shadow is much higher than the probability that the convex surface causes the upper shadow. The perception is described by the posterior distribution given as follows.
\begin{eqnarray*}
&&p(\text{Surface}|\text{Shadow}, \text{light above})=\frac{p(\text{Surface}, \text{Shadow}, \text{light above})}{p(\text{Shadow}, \text{light above})}\\
&&=\frac{p(\text{Shadow}|\text{Surface},\text{light above})p(\text{Surface}|\text{light above})p(\text{light above})}{p(\text{Shadow}|\text{light above})p(\text{light above})}\\
&&\propto p(\text{Shadow}|\text{Surface},\text{light above})p(\text{Surface}|\text{light above})
\end{eqnarray*}
The first line is an application of the definition of conditional probability, and the second line is the application of the product rule \cite{bishop:06}. Since the denominator is a normalising constant, we omit it using the symbol `$\propto$' (meaning `proportional to') in the third line. Given the assumption that surfaces and light sources are independent, i.e., $p(\text{Surface}|\text{light above})=p(\text{Surface})$, we have
\begin{eqnarray*}
p(\text{Surface}|\text{Shadow}, \text{light above})&\propto&p(\text{Shadow}|\text{Surface},\text{light above})p(\text{Surface}).
\end{eqnarray*}
We thus have the following values.
\begin{eqnarray*}
&&p(\text{Surface}=\text{concave}|\text{Shadow}=\text{upper}, \text{light above})\propto 0.9\cdot0.45=0.405\\
&&p(\text{Surface}=\text{convex}|\text{Shadow}=\text{upper}, \text{light above})\propto 0.1\cdot0.55=0.055
\end{eqnarray*}
Therefore, we have the following outcomes.
\begin{eqnarray*}
&&p(\text{Surface}=\text{concave}|\text{Shadow}=\text{upper}, \text{light above})=\frac{0.405}{0.405+0.055}\approx0.88\\
&&p(\text{Surface}=\text{convex}|\text{Shadow}=\text{upper}, \text{light above})=\frac{0.055}{0.405+0.055}\approx0.12
\end{eqnarray*}
The posterior shows that the probability of the concave surface is higher than the probability of the convex surface after observing the upper shadow.
\par
An increasing number of scientific experiments support this view of perception as Bayesian inference, which is rooted in Helmholtz's view of perception as unconscious inference \cite{Helmholtz:25}. However, a limitation of this approach, as well as free energy principle \cite{friston:10} and predictive coding \cite{knill:04}, is the disconnection between the two essential processes: on the one hand, the process by which the posterior is derived from the likelihood and prior, and on the other hand, the process by which the likelihood and prior are derived from data. The latter process, often referred to as parameter estimation or learning, is typically dealt with using maximum likelihood estimation (MLE) or maximum a posteriori estimation.
\begin{figure}[t]
\begin{tabular}{cc}
\begin{minipage}{0.59\hsize}
\begin{center}
\begin{tabular}{l|ll|l}
Model & Surface & Shadow & Data\\\hline
$m_{1}$ & 0 (concave) & 0 (lower) & $d_{1}$\\
$m_{2}$ & 0 (concave) & 1 (upper) & $d_{2},d_{3},d_{4}$\\
$m_{3}$ & 1 (convex) & 0 (lower) & $d_{5}, d_{6},d_{7},d_{8},d_{9}$\\
$m_{4}$ & 1 (convex) & 1 (upper) & $d_{10}$
\end{tabular}
\end{center}
\end{minipage}
\begin{minipage}{0.4\hsize}
\begin{center}
\includegraphics[scale=0.17]{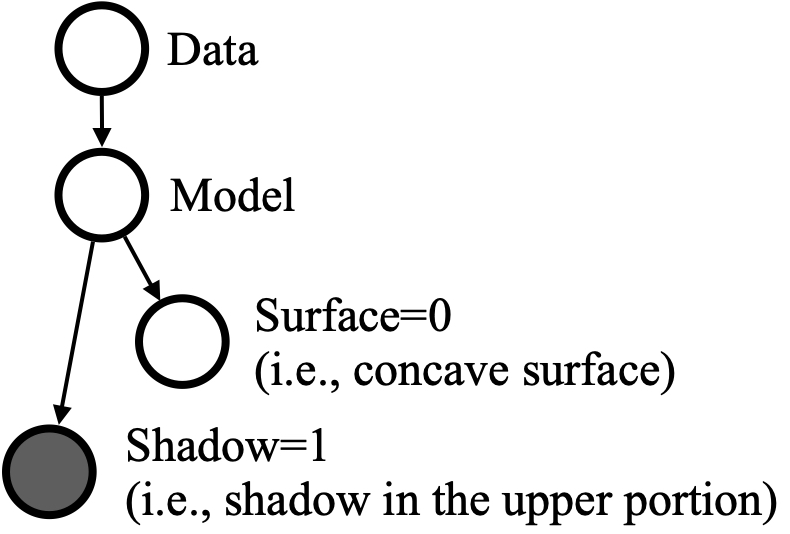}
\end{center}
\end{minipage}
\end{tabular}
\caption{Data (left) and graphical model (right) for perceptual reasoning.}
\label{PR}
\end{figure}
\par
Figure \ref{PR} summarises our approach applied to the perception problem. Each row of the left table shows a different state of the world built with the two random variables, {\it Surface} and {\it Shadow}. The right most column shows the data evidencing the models, which reflects our intuition that the lower (resp. upper) shadow often comes with convex (resp. concave) surface. This data is an input into the graphical model shown on the right hand side of the figure. The model, formally defined in the next section, shows the process by which data generate the perception. We will show that this model causes the following posterior.
\begin{eqnarray*}
p(concave|upper)&=&\frac{\sum_{n=1}^{4}p(concave|m_{n})p(upper|m_{n})\sum_{k=1}^{10}p(m_{n}|d_{k})p(d_{k})}{\sum_{n=1}^{4}p(upper|m_{n})\sum_{k=1}^{10}p(m_{n}|d_{k})p(d_{k})}\\
&=&\frac{\sum_{n=1}^{4}p(concave|m_{n})p(upper|m_{n})K_{n}}{\sum_{n=1}^{4}p(upper|m_{n})K_{n}}=\frac{K_{2}}{K_{2}+K_{4}}=\frac{3}{4}
\end{eqnarray*}
Here, $K_{n}$ denotes the number of data in the $n$-th model and $p(X|m_{n})=1$ if $X$ is the case in the $n$-th model and $0$ otherwise. We ommited `light above' for simplicity. The correctness of the result will be shown by the fact that our model and the existing approaches with MLE always give the same results. For example, given the data shown in Figure \ref{PR}, the approach used in Figure \ref{fig:illusion} gives
\begin{eqnarray*}
p(concave|upper)&=&\frac{p(upper|concave)p(concave)}{p(upper|concave)p(concave)+p(upper|convex)p(convex)}\\
&=&\frac{3/4\times 4/10}{3/4\times 4/10 + 1/6\times 6/10}=\frac{3}{4},
\end{eqnarray*}
where the value of each term was obtained using MLE.
\par
We will also show that the same model is applicable to logical reasoning. The left table in Figure \ref{LR} shows data evidencing the models built with two random variables, Rain and Wet, representing propositional symbols, rain and wet (abbreviated as $r$ and $w$ below), respectively. We will show that this model causes the following posterior.
\begin{eqnarray*}
p(w|r,r\to w)&=&\frac{\sum_{n=1}^{4}p(w|m_{n})p(r|m_{n})p(r\to w|m_{n})\sum_{k=1}^{10}p(m_{n}|d_{k})p(d_{k})}{\sum_{n=1}^{4}p(r|m_{n})p(r\to w|m_{n})\sum_{k=1}^{10}p(m_{n}|d_{k})p(d_{k})}\\
&=&\frac{\sum_{n=1}^{4}p(w|m_{n})p(r|m_{n})p(r\to w|m_{n})K_{n}}{\sum_{n=1}^{4}p(r|m_{n})p(r\to w|m_{n})K_{n}}=\frac{K_{4}}{K_{4}}=1
\end{eqnarray*}
The correctness of the model will be shown by the fact that a probability of one corresponds to a probabilistic improvement of the classical consequence relation.
\begin{figure}[t]
\begin{tabular}{cc}
\begin{minipage}{0.46\hsize}
\begin{center}
\begin{tabular}{l|ll|l}
Model & $Rain$ & $W\!et$ & Data\\\hline
$m_{1}$ & 0 (false) & 0 (false) & $d_{1},d_{2},d_{3},d_{4}$\\
$m_{2}$ & 0 (false) & 1 (true) & $d_{5},d_{6}$\\
$m_{3}$ & 1 (true) & 0 (false) & $d_{7}$\\
$m_{4}$ & 1 (true) & 1 (true) & $d_{8},d_{9},d_{10}$
\end{tabular}
\end{center}
\end{minipage}
\begin{minipage}{0.53\hsize}
\begin{center}
\includegraphics[scale=0.17]{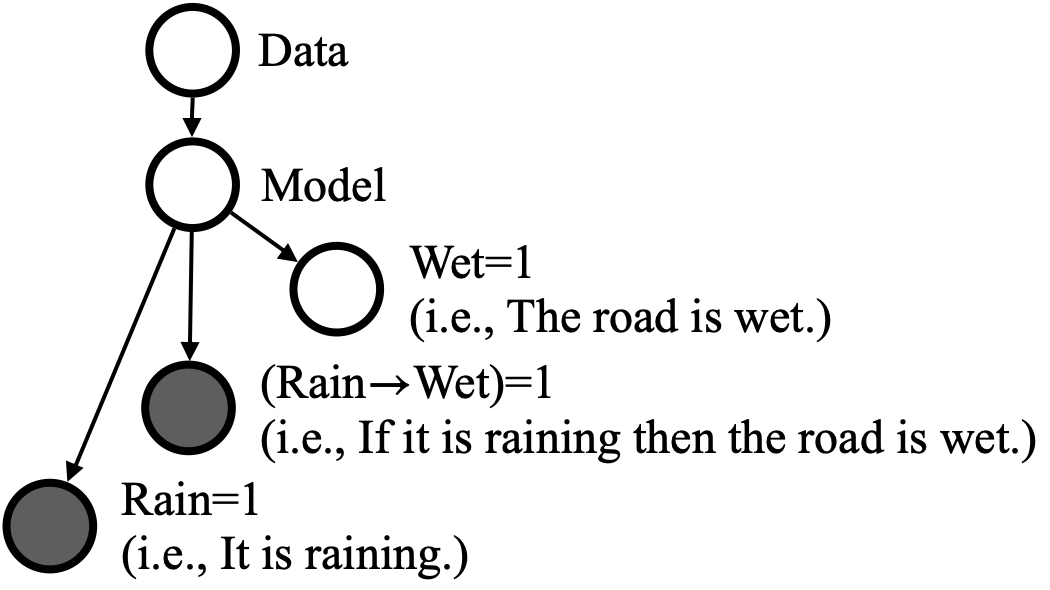}
\end{center}
\end{minipage}
\end{tabular}
\caption{Data (left) and graphical model (right) for logical reasoning.}
\label{LR}
\end{figure}
%
\section{Generative Logic Models}\label{sec:model}
Let ${\cal D}=\{d_{1},d_{2},...,d_{K}\}$ be a multiset of $K$ data. $D$ is a random variable whose realisations are data in ${\cal D}$. For all data $d_{k}\in{\cal D}$, we define the probability of $d_{k}$ as $p(D=d_{k})=1/K$.
\par
Let ${\cal L}$ represent a propositional or first-order language. For the sake of simplicity, we assume no function symbol or open formula in ${\cal L}$. Let ${\cal M}=\{m_{1},m_{2},...,m_{N}\}$ be the set of all $N$ models associated with ${\cal L}$. We assume ${\cal D}$ to be complete with respect to ${\cal M}$, meaning that each data in ${\cal D}$ belongs to a single model in ${\cal M}$. $m$ is a function that maps each data to such a single model. $M$ is a random variable whose realisations are models in ${\cal M}$. For all models $m_{n}\in{\cal M}$ and data $d_{k}\in{\cal D}$, we define the conditional probability of $m_{n}$ given $d_{k}$, as follows.
\begin{eqnarray*}
&&p(M=m_{n}|D=d_{k})=
\begin{cases}
1 & \text{if } m_{n}=m(d_{k})\\
0 & \text{otherwise }
\end{cases}
\end{eqnarray*}
Formal logic considers an interpretation on each model. The interpretation is a function that maps each formula to a truth value, which represents knowledge of the world. We here introduce parameter $\mu\in[0,1]$ to represent the extent to which each model is taken for granted in the interpretation. Concretely, $\mu$ denotes the probability that a formula is interpreted as being true (resp. false) in a model where it is true (resp. false). $1-\mu$ is therefore the probability that a formula is interpreted as being true (resp. false) in a model where it is false (resp. true). We assume that each formula is a random variable whose realisations are 0 and 1, denoting false and true, respectively. For all models $m_{n}\in{\cal M}$ and formulae $\alpha\in L$, we define the conditional probability of each truth value of $\alpha$ given $m_{n}$, as follows.
\begin{eqnarray*}
&&p(\alpha=1|M=m_{n})=
\begin{cases}
\mu & \text{if } m_{n}\in\llbracket\alpha=1\rrbracket\\
1-\mu & \text{otherwise }
\end{cases}
\\
&&p(\alpha=0|M=m_{n})=
\begin{cases}
\mu & \text{if } m_{n}\in\llbracket\alpha=0\rrbracket\\
1-\mu & \text{otherwise }
\end{cases}
\end{eqnarray*}
Here, $\llbracket\alpha=1\rrbracket$ denotes the set of all models in which $\alpha$ is true, and $\llbracket\alpha=0\rrbracket$ the set of all models in which $\alpha$ is false. The above expressions can be simply written as a Bernoulli distribution with parameter $\mu\in[0,1]$, i.e.,
\begin{eqnarray*}
p(\alpha|M=m_{n})=\mu^{\llbracket\alpha\rrbracket_{m_{n}}}(1-\mu)^{1-\llbracket\alpha\rrbracket_{m_{n}}}.
\end{eqnarray*}
Here, $\llbracket\alpha\rrbracket_{m_{n}}=1$ if $m_{n}\in\llbracket\alpha\rrbracket$ and $\llbracket\alpha\rrbracket_{m_{n}}=0$ otherwise. Recall that $\llbracket\alpha\rrbracket_{m_{n}}$ is either $\llbracket\alpha=0\rrbracket_{m_{n}}$ or $\llbracket\alpha=1\rrbracket_{m_{n}}$ as $\alpha$ is a random variable.
%
%
In classical logic, given a model, the truth value of each formula is independently determined. In probability theory, this means that the truth values of any two formulae $\alpha_{1}$ and $\alpha_{2}$ are conditionally independent given a model $m_{n}$, i.e., $p(\alpha_{1},\alpha_{2}|M=m_{n})=p(\alpha_{1}|M=m_{n})p(\alpha_{2}|M=m_{n})$. Note that the conditional independence holds not only for atomic formulae but for compound formulae as well.\footnote{In contrast, independence, i.e., $p(\alpha_{1},\alpha_{2})= p(\alpha_{1})p(\alpha_{2})$, generally holds for neither atomic formulae nor compound formulae.} Let $\Gamma=\{\alpha_{1},\alpha_{2},...,\alpha_{J}\}$ be a multiset of $J$ formulae. We thus have
\begin{eqnarray*}
p(\Gamma|M=m_{n})=\prod_{j=1}^{J}p(\alpha_{j}|M=m_{n}).
\end{eqnarray*}
Thus far, we have defined $p(D)$ and $p(M|D)$ as categorical distributions and $p(\Gamma|M)$ as Bernoulli distributions with parameter $\mu$. Given a value of the parameter $\mu$, they provide the full joint distribution over all of the random variables, i.e. $p(\Gamma,M,D)$. We call $\{p(\Gamma|M,\mu), p(M|D), p(D)\}$ a generative logical model. In sum, the generative logical model defines the probabilistic process by which data generates models and then the models generates the truth values of formulae. The generative logical model meets the following logical property.
\begin{proposition}\label{kolmogorov}
The generative logical model satisfies the Kolmogorov's axioms.
\end{proposition}
\begin{proof}
See \cite{Kido:22}.
\qed
\end{proof}
From Proposition \ref{kolmogorov}, we can show that $p(\alpha=0)=p(\neg\alpha=1)$, for all $\alpha\in{\cal L}$. We therefore replace $\alpha=0$ by $\lnot\alpha=1$ and then abbreviate $\lnot\alpha=1$ to $\lnot\alpha$. We also abbreviate $M=m_{n}$ to $m_{n}$ and $D=d_{k}$ to $d_{k}$.
%
\begin{example}\label{ex:propositional}
Let $\{p(\Gamma|M,\mu=1), p(M|D), p(D)\}$ be a generative logic model constructed with the models and data shown in Figure \ref{LR}. The following is an example of probabilistic reasoning on propositional formulae.
\begin{eqnarray*}
p(w|r)&=&\frac{\sum_{n=1}^{4}\sum_{k=1}^{10}p(w,r,m_{n},d_{k})}{\sum_{n=1}^{4}\sum_{k=1}^{10}p(r,m_{n},d_{k})}=\frac{\sum_{n=1}^{4}p(w|m_{n})p(r|m_{n})\sum_{k=1}^{10}p(m_{n}|d_{k})p(d_{k})}{\sum_{n=1}^{4}p(r|m_{n})\sum_{k=1}^{10}p(m_{n}|d_{k})p(d_{k})}\\
&=&\frac{\sum_{n=1}^{4}\llbracket w\rrbracket_{m_{n}}\llbracket r\rrbracket_{m_{n}}\sum_{k=1}^{10}p(m_{n}|d_{k})}{\sum_{n=1}^{4}\llbracket r\rrbracket_{m_{n}}\sum_{k=1}^{10}p(m_{n}|d_{k})}=\frac{\sum_{k=1}^{10}\llbracket w\rrbracket_{m(d_{k})}\llbracket r\rrbracket_{m(d_{k})}}{\sum_{k=1}^{10}\llbracket r\rrbracket_{m(d_{k})}}=\frac{3}{4}
\end{eqnarray*}
Here, $r$ and $w$ denote $rain$ and $wet$, respectively. In line 2, we eliminated the summation over models based on $\sum_{n}p(\alpha|m_{n})p(m_{n}|d_{k})=\ms{\alpha}_{m(d_{k})}$. This is an important technique as it causes the linear complexity with respect to the number of data, i.e., $O(K)$, which is independent from the number of models exponentially increasing with respect to the number of symbols in propositional logic and unbounded in first-order logic.
\end{example}
\section{Correctness}
\subsection{Perceptual Reasoning}\label{perceptual}
Fenstad's theorems \cite{fenstad:67} tell us how to properly combine first-order logic and probability theory. For all propositional formulae $\alpha\in{\cal L}$, they imply that $p(\alpha)$ should satisfy
\begin{eqnarray}\label{Fenstad}
p(\alpha)=\sum_{n: m_{n}\in\llbracket\alpha\rrbracket}^{N}p(m_{n}).
\end{eqnarray}
\begin{example}
Given the models and data shown in Figure \ref{LR}, $p(rain)=p(m_{3})+p(m_{4})$ and $p(rain\to wet)=p(m_{1})+p(m_{2})+p(m_{4})$.
\end{example}
Here, how should one define the probability distribution $p(M)$ over models? MLE is the most widely used method to estimate a probability distribution only from data. Given i.i.d. (independent and identically distributed) data ${\cal D}$, MLE defines $p(M)$ as the parameter $\Phi$ maximising the posterior distribution over ${\cal D}$, i.e., $p(M) =\argmax_{\Phi}p({\cal D}|\Phi)=\argmax_{\Phi}\prod_{d_{k}\in{\cal D}}p(d_{k}|\Phi).$
\par
Solving the simultaneous equations of the log likelihood $\log\prod_{d_{k}\in{\cal D}}p(d_{k}|\Phi)$, the maximum likelihood estimate turns out to be
\begin{eqnarray}\label{MLE}
p(M)=\left(\frac{K_{1}}{K},\frac{K_{2}}{K},...,\frac{K_{N}}{K}\right).
\end{eqnarray}
Therefore, the maximum likelihood estimate for the $n$-th model is just the ratio of the number of data in the model to the total number of data. Combining Equations (\ref{Fenstad}) and (\ref{MLE}), we have 
\begin{eqnarray}\label{combined}
p(\alpha)=\sum_{n: m_{n}\in\llbracket\alpha\rrbracket}^{N}\frac{K_{n}}{K},
\end{eqnarray}
where $K$ and $K_{n}$ denote the number of data and the number of data in the $n$-th model, respectively.
\begin{example}
Given the models and data shown in Figure \ref{LR}, the maximum likelihood estimate gives $p(M)=\left(\frac{4}{10},\frac{2}{10},\frac{1}{10},\frac{3}{10}\right)$. Therefore, $p(rain)=p(m_{3})+p(m_{4})=\frac{4}{10}$ and $p(rain\to wet)=p(m_{1})+p(m_{2})+p(m_{4})=\frac{9}{10}$.
\end{example}
Now, let $\{p(\Gamma|M,\mu=1),p(M|D),p(D)\}$ be a generative logic model with $\mu=1$. We can show that the model satisfies Equation (\ref{combined}).
\begin{eqnarray*}
p(\alpha)&=&\sum_{n=1}^{N}\sum_{k=1}^{K}p(\alpha|m_{n})p(m_{n}|d_{k})p(d_{k})=\sum_{n=1}^{N}p(\alpha|m_{n})\sum_{k=1}^{K}p(m_{n}|d_{k})p(d_{k})\\
&=&\sum_{n=1}^{N}[\![\alpha]\!]_{m_{n}}\frac{K_{n}}{K}=\sum_{n~such~that~m_{n}\in[\![\alpha]\!]}^{N}\frac{K_{n}}{K}
\end{eqnarray*}
\par
Moreover, we can show that approaches to reason knowledge from knowledge using Bayes' theorem (i.e., $p(\alpha|\Delta)=p(\Delta|\alpha)p(\alpha)/p(\Delta)$ as illustrated in Section \ref{sec:motivation}) can be replaced by the application of the generative logic model. Each term of the Bayes' theorem is developed as follows.
\begin{eqnarray*}
p(\alpha|\Delta)=\frac{\sum_{m}p(\alpha|m)p(\Delta|m)p(m)}{\sum_{m}p(\Delta|m)p(m)}&~~~~&p(\Delta|\alpha)=\frac{\sum_{m}p(\Delta|m)p(\alpha|m)p(m)}{\sum_{m}p(\alpha|m)p(m)}\\
p(\alpha)=\sum_{m}p(\alpha|m)p(m)&~~~~&p(\Delta)=\sum_{m}p(\Delta|m)p(m)
\end{eqnarray*}
The following result shows that the Bayes' theorem still holds.
\begin{eqnarray*}
\frac{p(\Delta|\alpha)p(\alpha)}{p(\Delta)}&=&\frac{\frac{\sum_{m}p(\Delta|m)p(\alpha|m)p(m)}{\sum_{m}p(\alpha|m)p(m)}\sum_{m}p(\alpha|m)p(m)}{\sum_{m}p(\Delta|m)p(m)}=p(\alpha|\Delta)
\end{eqnarray*}
%
\subsection{Classical consequence relation}\label{section:logic}
We showed in the last section that, given $\{p(\Gamma|M,\mu=1),p(M|D),p(D)\}$, $p(M)$ is equivalent to the maximum likelihood estimate, i.e., for all $m_{n}\in{\cal M}$,
\begin{eqnarray*}
p(m_{n})=\sum_{k=1}^{K}p(m_{n}|d_{k})p(d_{k})=\frac{K_{n}}{K}.
\end{eqnarray*}
Therefore, $\{p(\Gamma|M,\mu=1),p(M|D),p(D)\}$ is equivalent to $\{p(\Gamma|M,\mu=1),p(M)\}$ when $p(M)$ is the maximum likelihood estimate. For the sake of simplicity, we also call the latter a generative logic model and use it without distinction.
\par
Let $\{p(M),p(\Gamma|M,\mu)\}$ be a generative logic model. This section looks at the generative logic model with two assumptions. The first assumption referred to as the model assumption is that all the models are possible, i.e., $0\notin p(M)$. The second assumption referred to as the Boolean assumption is that the likelihood of each formula being true is the truth value of the formula, i.e., $\mu=1$. The following theorem relates probabilistic reasoning on logical formulae to the probability of the models of the logical formulae.
\begin{theorem}\label{thrm:lr}
Let $\{p(M),p(\Gamma|M,\mu)\}$ be a generative logic model such that $0\notin p(M)$ and $\mu=1$. For all $\alpha\in\Gamma$ and $\Delta\subseteq\Gamma$,
\begin{eqnarray*}
p(\alpha|\Delta)=
\begin{cases}
\displaystyle{\frac{\sum_{m\in\ms{\Delta}\cap\ms{\alpha}}p(m)}{\sum_{m\in \ms{\Delta}}p(m)}}&\text{if }\ms{\Delta}\neq\emptyset\\
\text{undefined}&\text{otherwise.}
\end{cases}
\end{eqnarray*}
\end{theorem}
\begin{proof}
See \cite{Kido:22}.
\qed
\end{proof}
Let $\alpha\in{\cal L}$ and $\Delta\subseteq{\cal L}$ such that $\ms{\Delta}\neq\emptyset$. From Theorem \ref{thrm:lr}, $p(\alpha|\Delta)=1$ if and only if $\Delta\cent\alpha$, i.e., $\ms{\Delta}\subseteq\ms{\alpha}$.
%
\subsection{Consequence relation}\label{section:commonsense}
Let $\{p(M),p(\Gamma|M,\mu)\}$ be a generative logic model. This section looks at the generative logic model with the Boolean assumption, i.e., $\mu=1$. This section is thus a full generalisation of the previous section discussing the generative logic models with the model assumption and the Boolean assumption. We distinguish possible models from impossible ones for a logical characterisation of the generative logic model.
\begin{definition}[Possible models]
Let $m\in{\cal M}$ and $\Delta\subseteq{\cal L}$. $m$ is a possible model of $\Delta$ if $m\in\ms{\Delta}$ and $p(m)\neq0$. 
\end{definition}
Models that are not possible are referred to as impossible. We use symbol $\pms{\Delta}$ to denote the set of all the possible models of $\Delta$, i.e., $\pms{\Delta}=\{m\in\m{\Delta}|p(m)\neq 0\}$. We assume $\pms{\Delta}_{m}=1$ if $m\in\pms{\Delta}$ and $\pms{\Delta}_{m}=0$ otherwise. Obviously, $\pms{\Delta}\subseteq\m{\Delta}$ holds, for all $\Delta\subseteq {\cal L}$, and $\pms{\Delta}=\m{\Delta}$ holds under the model assumption, i.e., $0\notin p(M)$. We define another consequence relation based only on possible models.
\begin{definition}[Consequence]
Let $\Delta\subseteq {\cal L}$ and $\alpha\in {\cal L}$. $\alpha$ is a consequence of $\Delta$, denoted by $\Delta\ent\alpha$, if $\pms{\alpha}\supseteq\pms{\Delta}$.
\end{definition}
Note that we refer to $\cent$ as the classical consequence relation and $\ent$ as the consequence relation. $\Delta\cent\alpha$ thus means that $\alpha$ is a classical consequence of $\Delta$. The following proposition shows that the consequence relation is weaker than the classical consequence relation.
\begin{proposition}
Let $\Delta\subseteq {\cal L}$ and $\alpha\in {\cal L}$. If $\Delta\cent\alpha$ then $\Delta\ent\alpha$, but not vice versa.
\end{proposition}
\begin{proof}
($\Rightarrow$) Recall that $\Delta\cent\alpha$ iff $\ms{\Delta}\subseteq\ms{\alpha}$. For all $X\subseteq{\cal M}$, $\ms{\Delta}\setminus X\subseteq\ms{\alpha}\setminus X$ holds. ($\Leftarrow$) Suppose $\Delta$, $\alpha$ and $m$ such that $\ms{\Delta}=\ms{\alpha}\cup\{m\}$ and $p(m)=0$. $\Delta\ent\alpha$ holds, but $\Delta\not\cent\alpha$.
\qed
\end{proof}
The difference between the consequence and the classical consequence is straightforward. For all models $m$, the classical consequence relation requires all the models of the premises to be the models of the conclusion. In contrast, the consequence relation requires all the possible models of the premises to be the models of the conclusion. The following theorem relates probabilistic reasoning on logical formulae to the probability of the possible models of the logical formulae.
\begin{theorem}\label{thrm:csr}
Let $\{p(M),p(\Gamma|M,\mu)\}$ be a generative logic model such that $\mu=1$. For all $\alpha\in\Gamma$ and $\Delta\subseteq\Gamma$,
\begin{eqnarray*}
p(\alpha|\Delta)=
\begin{cases}
\displaystyle{\frac{\sum_{m\in\pms{\Delta}\cap\pms{\alpha}}p(m)}{\sum_{m\in\pms{\Delta}}p(m)}}&\text{if }\pms{\Delta}\neq\emptyset\\
\text{undefined}&\text{otherwise.}
\end{cases}
\end{eqnarray*}
\end{theorem}
\begin{proof}
Dividing models into the possible models $\pms{\Delta}$ and the others, we have
\begin{eqnarray*}
p(\alpha|\Delta)&=&\frac{\sum_{m}p(\alpha|m)p(\Delta|m)p(m)}{\sum_{m}p(\Delta|m)p(m)}\\
&=&\frac{\sum_{\hat{m}\in\pms{\Delta}}p(\alpha|\hat{m})\mu^{|\Delta|}p(\hat{m})+\sum_{m\notin\pms{\Delta}}p(\alpha|m)p(\Delta|m)p(m)}{\sum_{\hat{m}\in\pms{\Delta}}\mu^{|\Delta|}p(\hat{m})+\sum_{m\notin\pms{\Delta}}p(\Delta|m)p(m)}.
\end{eqnarray*}
$p(\Delta|m)=\prod_{\beta\in\Delta}p(\beta|m)=\prod_{\beta\in\Delta}\mu^{\m{\beta}_{m}}(1-\mu)^{1-{\m{\beta}_{m}}}$. Recall that $\pms{\Delta}=\{m\in\m{\Delta}|p(m)\neq0\}$. Thus, for all $m\notin\pms{\Delta}$, if $m\notin\m{\Delta}$ then there is $\beta\in\Delta$ such that $\m{\beta}_{m}=0$ and if $m\in\m{\Delta}$ then $p(m)=0$. Therefore, $p(\Delta|m)=0$ or $p(m)=0$ when $\mu=1$, for all $m\notin\pms{\Delta}$. We thus have
\begin{eqnarray*}
p(\alpha|\Delta)=\frac{\sum_{m\in\pms{\Delta}}p(\alpha|m)1^{|\Delta|}p(m)}{\sum_{m\in\pms{\Delta}}1^{|\Delta|}p(m)}=\frac{\sum_{m\in\pms{\Delta}}1^{\m{\alpha}_{m}}0^{1-\m{\alpha}_{m}}p(m)}{\sum_{m\in\pms{\Delta}}p(m)}.
\end{eqnarray*}
Since $1^{\m{\alpha}_{m}}0^{1-\m{\alpha}_{m}}=1^{1}0^{0}=1$ if $m\in\m{\alpha}$ and $1^{\m{\alpha}_{m}}0^{1-\m{\alpha}_{m}}=1^{0}0^{1}=0$ if $m\notin\m{\alpha}$, we have
\begin{eqnarray}\label{proof:2}
p(\alpha|\Delta)=\frac{\sum_{m\in\pms{\Delta}\cap\m{\alpha}}p(m)}{\sum_{m\in\pms{\Delta}}p(m)}=\frac{\sum_{m\in\pms{\Delta}\cap\pms{\alpha}}p(m)}{\sum_{m\in\pms{\Delta}}p(m)}.
\end{eqnarray}
In addition, if $\pms{\Delta}=\emptyset$ then $p(\alpha|\Delta)$ is undefined due to division by zero.
\qed
\end{proof}
We can now relate probabilistic reasoning on the generative logic model to the consequence relation, i.e., $\ent$.
\begin{corollary}\label{cor:2}
Let $\{p(M),p(\Gamma|M,\mu)\}$ be a generative logic model such that $\mu=1$. For all $\alpha\in\Gamma$ and $\Delta\subseteq\Gamma$ such that $\pms{\Delta}\neq\emptyset$, $p(\alpha|\Delta)=1$ iff $\Delta\ent\alpha$.
\end{corollary}
\begin{proof}
Recall that $\Delta\ent\alpha$ iff $\pms{\alpha}\supseteq\pms{\Delta}$. Since Equation (\ref{proof:2}) and $p(m)\neq0$, for all $m\in\pms{\Delta}$, $p(\alpha|\Delta)=1$ iff $\pms{\alpha}\supseteq\pms{\Delta}$.
\qed
\end{proof}
Given $\pms{\Delta}=\emptyset$, $p(\alpha|\Delta)$ is undefined due to division by zero, whereas $\Delta\ent\alpha$ holds, for all $\alpha$, by definition. Corollary \ref{cor:2} thus does not hold for $\pms{\Delta}=\emptyset$.
%
\subsection{Classical consequences with maximal consistent sets}\label{section:paraconsistency}
Let $\{p(M),p(\Gamma|M,\mu)\}$ be a generative logic model. This section looks at the generative logic model with two assumptions. The first assumption is the model assumption, i.e., $0\notin p(M)$. The second assumption referred to as the limit assumption is that the likelihood of each formula being true approaches the truth value of the formula, i.e., $\mu\rightarrow 1$. To explain why we need limits in the interpretation of formal logic, consider $\alpha,\beta\in{\cal L}$ such that $\ms{\beta}=\emptyset$. We have
\begin{eqnarray*}
p(\alpha|\beta)=\frac{\sum_{m}p(\alpha|m)p(\beta|m)p(m)}{\sum_{m}p(\beta|m)p(m)}=\frac{\sum_{m}p(\alpha|m)(1-\mu)p(m)}{\sum_{m}(1-\mu)p(m)}.
\end{eqnarray*}
As discussed before, we must assume $\mu=1$ if we conform to the interpretation of formal logic. However, as seen in the previous sections, this causes a probability undefined due to division by zero. Given $\mu\neq 1$, however, we have
\begin{eqnarray*}
p(\alpha|\beta)=\frac{\sum_{m}p(\alpha|m)(1-\mu)p(m)}{\sum_{m}(1-\mu)p(m)}=\frac{\sum_{m}p(\alpha|m)p(m)}{\sum_{m}p(m)}=p(\alpha).
\end{eqnarray*}
Amongst $\mu\neq 1$, we should assume $\mu\to1$ because the expression below shows that only $\mu\to1$ results in $p(\alpha)$ given by $\mu=1$.
\begin{eqnarray*}
p(\alpha)=\sum_{m}p(\alpha|m)p(m)=\sum_{m}\mu^{\ms{\alpha}_{m}}(1-\mu)^{1-\ms{\alpha}_{m}}p(m).
\end{eqnarray*}
\begin{example}
Let us see how limits work in practice. Consider the three conditional probabilities given different inconsistent premises, shown on the right graph in Figure \ref{limit}. Given the probability distribution over models shown on the left table in Figure \ref{limit}, the first expression can be expanded as follows.
\begin{eqnarray*}
&&p(rain|rain,wet,\lnot wet)=\frac{\sum_{m}p(rain|m)^{2}p(wet|m)p(\lnot wet|m)p(m)}{\sum_{m}p(rain|m)p(wet|m)p(\lnot wet|m)p(m)}\\
&=&\frac{\mu(1-\mu)^{3}p(m_{1})+\mu(1-\mu)^{3}p(m_{2})+\mu^{3}(1-\mu)p(m_{3})+\mu^{3}(1-\mu)p(m_{4})}{\mu(1-\mu)^{2}p(m_{1})+\mu(1-\mu)^{2}p(m_{2})+\mu^{2}(1-\mu)p(m_{3})+\mu^{2}(1-\mu)p(m_{4})}\\
&=&\frac{0.4\mu(1-\mu)^{3}+0.2\mu(1-\mu)^{3}+0.1\mu^{3}(1-\mu)+0.3\mu^{3}(1-\mu)}{0.4\mu(1-\mu)^{2}+0.2\mu(1-\mu)^{2}+0.1\mu^{2}(1-\mu)+0.3\mu^{2}(1-\mu)}
\end{eqnarray*}
The right graph in Figure \ref{limit} shows $p(rain|rain,wet,\lnot wet)$ given different $\mu$ values. The graph also includes the other two conditional probabilities calculated in the same manner. Each of the open circles represents being undefined. This means that no substitution gives a probability, even though the curve approaches a certain probability. The certain probability can only be obtained by the use of limits. Indeed, given $\mu\to1$, the three conditional probabilities turn out to be 1, 0.5 and 0.4, respectively.
\end{example}
\begin{figure}[t]
\begin{tabular}{cc}
 \begin{minipage}{0.29\hsize}
\begin{center}
\begin{tabular}{c|cc|c}
 & $rain$ & $wet$ & $p(M)$\\\hline
$m_{1}$ & 0 & 0 & 0.4\\
$m_{2}$ & 0 & 1 & 0.2\\
$m_{3}$ & 1 & 0 & 0.1\\
$m_{4}$ & 1 & 1 & 0.3
\end{tabular}
\end{center}
 \end{minipage}
 \begin{minipage}{0.7\hsize}
 \begin{center}
\includegraphics[scale=0.16]{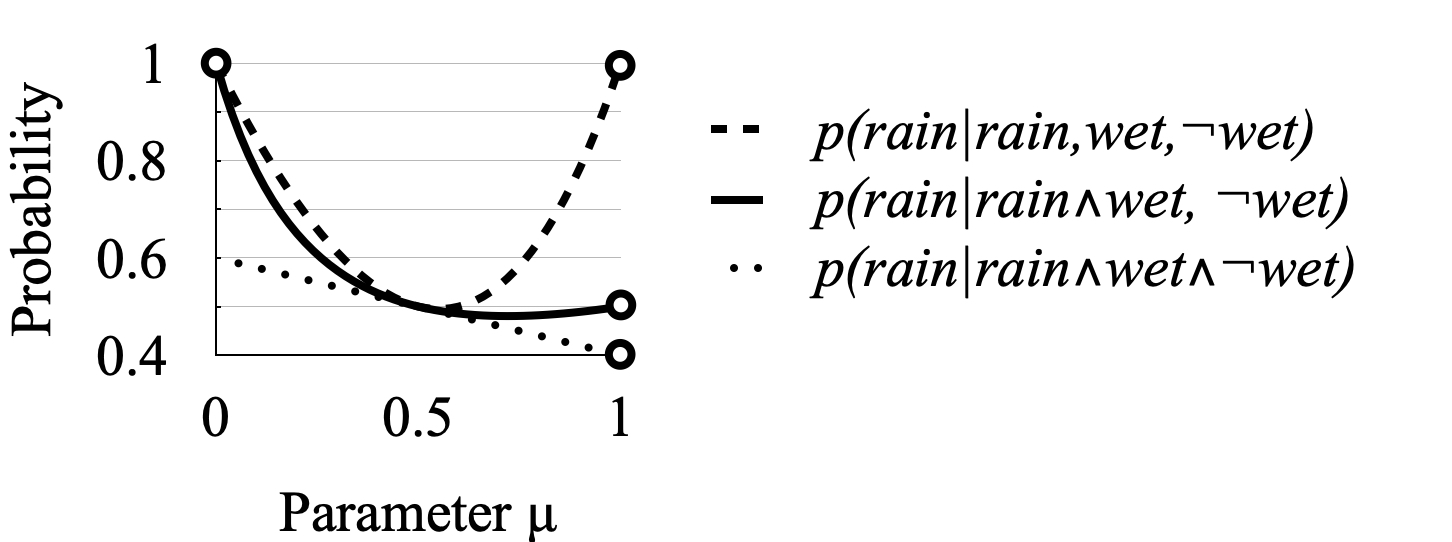}
 \end{center}
 \end{minipage}
\end{tabular}
\caption{The model distribution (left) and $\mu\in[0,1]$ versus the conditional probabilities given three different inconsistent sets of formulae (right).}
\label{limit}
\end{figure}
\par
We use maximal consistent sets to characterise the generative logic model.
\begin{definition}[Maximal consistent subsets]
Let $S,\Delta\subseteq{\cal L}$. $S$ is a maximal consistent subset of $\Delta$ if $S\subseteq\Delta$, $\ms{S}\neq\emptyset$ and $\ms{S\cup\{\alpha\}}=\emptyset$, for all $\alpha\in\Delta\setminus S$.
\end{definition}
We refer to a maximal consistent subset as a cardinality-maximal consistent subset when the set has the maximum cardinality. We use symbol $MCS(\Delta)$ to denote the set of the cardinality-maximal consistent subsets of $\Delta\subseteq {\cal L}$.
\begin{example}[Cardinality-maximal consistent sets]\label{ex:MCS}
Consider $\Delta=\{$ $rain$, $wet$, $rain\to wet$, $\lnot wet\}$. There are three maximal consistent subsets of $\Delta$: $S_{1}=\{rain,wet,rain\to wet\}$, $S_{2}=\{rain,\lnot wet\}$, and $S_{3}=\{rain\to wet,\lnot wet\}$. Only $S_{1}$ is the cardinality-maximal consistent subset of $\Delta$, i.e., $MCS(\Delta)=\{S_{1}\}$.
\end{example}
We define approximate models as the models of the cardinality-maximal consistent sets.
\begin{definition}[Approximate models]
Let $m\in{\cal M}$ and $\Delta\subseteq{\cal L}$. $m$ is an approximate model of $\Delta$ if there is $S\subseteq\Delta$ such that $S$ is a cardinality-maximal consistent subset of $\Delta$ and $m\in\ms{S}$.
\end{definition}
We use symbol $\ams{\Delta}$ to denote the approximate models of $\Delta$. In short, $\ams{\Delta}=\bigcup_{S\in MCS(\Delta)}\ms{S}$.
\begin{example}[Approximate models]\label{ex:AM}
Consider Example \ref{ex:MCS}. $m_{4}$ is the only model of the cardinality-maximal consistent subset $S_{1}$. Therefore, the approximate model of $\Delta$ is $m_{4}$. In short, $\ams{\Delta}=\bigcup_{S\in MCS(\Delta)}\ms{S}=\{m_{4}\}$.
\end{example}
The following theorem relates probabilistic reasoning on logical formulae to the probability of the approximate models of the logical formulae.
\begin{theorem}\label{thrm:pr}
Let $\{p(M),p(\Gamma|M,\mu)\}$ be a generative logic model such that $0\notin p(M)$ and $\mu\to 1$. For all $\alpha\in\Gamma$ and $\Delta\subseteq\Gamma$ such that $\ams{\Delta}\neq\emptyset$,
\begin{eqnarray*}
p(\alpha|\Delta)=\frac{\sum_{m\in\ams{\Delta}\cap\ms{\alpha}}p(m)}{\sum_{m\in \ams{\Delta}}p(m)}.
\end{eqnarray*}
\end{theorem}
\begin{proof}
See \cite{Kido:22}.
\qed
\end{proof}
We can now relate probabilistic reasoning on the generative logic model to the classical consequence relation, i.e., $\cent$, with maximal consistent sets.
\begin{corollary}
Let $\{p(M),p(\Gamma|M,\mu)\}$ be a generative logic model such that $0\notin p(M)$ and $\mu\to 1$. For all $\alpha\in\Gamma$ and $\Delta\subseteq\Gamma$ such that $\ams{\Delta}\neq\emptyset$, $p(\alpha|\Delta)=1$ iff $S\cent\alpha$, for all cardinality-maximal consistent subsets $S$ of $\Delta$.
\end{corollary}
\begin{proof}
See \cite{Kido:22}.
\qed
\end{proof}
The following theorem holds when $\ams{\Delta}=\emptyset$.
\begin{theorem}\label{thrm:pr2}
Let $\{p(M),p(\Gamma|M,\mu)\}$ be a generative logic model such that $0\notin p(M)$ and $\mu\to 1$. For all $\alpha\in {\cal L}$ and $\Delta\subseteq {\cal L}$ such that $\mm{\Delta}=\emptyset$, $p(\alpha|\Delta)=p(\alpha)$.
\end{theorem}
\begin{proof}
Since $\mm{\Delta}=\emptyset$, we have
\begin{eqnarray*}
p(\alpha|\Delta)=\lim_{\mu\rightarrow 1}\frac{\sum_{m}p(\alpha|m)p(m)p(\Delta|m)}{\sum_{m}p(m)p(\Delta|m)}=\lim_{\mu\rightarrow 1}\frac{\sum_{m\notin\mm{\Delta}}p(\alpha|m)p(m)p(\Delta|m)}{\sum_{m\notin\mm{\Delta}}p(m)p(\Delta|m)}.
\end{eqnarray*}
Now, $\mm{\Delta}=\bigcup_{S\in MCS(\Delta)}\m{S}=\emptyset$ iff no model satisfies any single formula in $\Delta$. Therefore, we have
\begin{eqnarray*}
p(\Delta|m)&=&\prod_{\beta\in\Delta}p(\beta|m)=\prod_{\beta\in\Delta}\mu^{\m{\beta}_{m}}(1-\mu)^{1-\m{\beta}_{m}}\\
&=&\mu^{\sum_{\beta\in\Delta}\m{\beta}_{m}}(1-\mu)^{\sum_{\beta\in\Delta}(1-\m{\beta}_{m})}=\mu^{0}(1-\mu)^{|\Delta|}.
\end{eqnarray*}
Therefore, we have
\begin{eqnarray*}
p(\alpha|\Delta)&=&\lim_{\mu\rightarrow 1}\frac{\sum_{m\notin\mm{\Delta}}p(\alpha|m)p(m)\mu^{0}(1-\mu)^{|\Delta|}}{\sum_{m\notin\mm{\Delta}}p(m)\mu^{0}(1-\mu)^{|\Delta|}}=\lim_{\mu\rightarrow 1}\frac{\sum_{m\notin\mm{\Delta}}p(\alpha|m)p(m)}{\sum_{m\notin\mm{\Delta}}p(m)}\\
&=&\lim_{\mu\rightarrow 1}\frac{\sum_{m}p(\alpha|m)p(m)}{\sum_{m}p(m)}=p(\alpha).
\end{eqnarray*}
\qed
\end{proof}
%
\subsection{Consequences with maximal possible sets}\label{section:counterfactuals}
Let $\{p(M),p(\Gamma|M,\mu)\}$ be a generative logic model. This section looks at the generative model under the limit assumption, i.e., $\mu\rightarrow 1$. This section is thus a full generalisation of the previous section discussing the generative logic models with the model assumption and the limit assumption. We introduce maximal possible sets to characterise the generative logic model.
\begin{definition}[Maximal possible subsets]
Let $S,\Delta\subseteq{\cal L}$. $S$ is a maximal possible subset of $\Delta$ if $S\subseteq\Delta$, $\pms{S}\neq\emptyset$ and $\pms{S\cup\{\alpha\}}=\emptyset$, for all $\alpha\in\Delta\setminus S$.
\end{definition}
We refer to a maximal possible subset as a cardinality-maximal possible subset when the set has the maximum cardinality. We use symbol $MPS(\Delta)$ to denote the set of the cardinality-maximal possible subsets of $\Delta\subseteq {\cal L}$.
\begin{example}[Cardinality-maximal possible sets]\label{ex:MPS}
Suppose the probability distribution $p(M)=(m_{1},m_{2},m_{3},m_{4})=(0.9,0.1,0,0)$ in Figure \ref{limit}. Consider again $\Delta=\{$ $rain$, $wet$, $rain\to wet,\lnot wet\}$. There are now two maximal possible subsets of $\Delta$: $S_{1}=\{wet,rain\to wet\}$ and $S_{2}=\{rain\to wet,\lnot wet\}$. Both $S_{1}$ and $S_{2}$ are the cardinality-maximal possible subsets of $\Delta$. Namely, $MPS(\Delta)=\{S_{1},S_{2}\}$. Note that neither $S_{3}=\{rain,rain\to wet, wet\}$ nor $S_{4}=\{rain,\lnot wet\}$ is a maximal possible subset of $\Delta$ although they are both maximal consistent subsets of $\Delta$.
\end{example}
We define possible approximate models as the possible models of the cardinality-maximal possible sets.
\begin{definition}[Possible approximate models]
Let $m\in{\cal M}$ and $\Delta\subseteq{\cal L}$. $m$ is a possible approximate model of $\Delta$ if there is $S\subseteq\Delta$ such that $S$ is a cardinality-maximal possible subset of $\Delta$ and $m\in\pms{S}$.
\end{definition}
We use symbol $\pams{\Delta}$ to denote the possible approximate models of $\Delta$. In short, $\pams{\Delta}=\bigcup_{S\in MPS(\Delta)}\pms{S}$.
\begin{example}[Possible approximate models]
Consider Example \ref{ex:MPS}. Only $m_{2}$ is the possible model of $S_{1}$ and $m_{1}$ is the possible model of $S_{2}$. Namely, $\pms{S_{1}}=\{m_{2}\}$ and $\pms{S_{2}}=\{m_{1}\}$. Therefore, $m_{1}$ and $m_{2}$ are the possible approximate models of $\Delta$, i.e., $\pams{\Delta}=\bigcup_{S\in MPS(\Delta)}\pms{S}=\{m_{1},m_{2}\}$.
\end{example}
The following theorem relates probabilistic reasoning on logical formulae to the possible approximate models of the logical formulae.
\begin{theorem}\label{thrm:cfr1}
Let $\{p(M),p(\Gamma|M,\mu)\}$ be a generative logic model such that $\mu\to 1$. For all $\alpha\in\Gamma$ and $\Delta\subseteq\Gamma$ such that $\pams{\Delta}\neq\emptyset$,
\begin{eqnarray*}
p(\alpha|\Delta)=\frac{\sum_{m\in\pams{\Delta}\cap\pms{\alpha}}p(m)}{\sum_{m\in\pams{\Delta}}p(m)}.
\end{eqnarray*}
\end{theorem}
\begin{proof}
We again use symbol $|\Delta|$ to denote the number of formulae in $\Delta$ and symbol $|\Delta|_{m}$ to denote the number of formulae in $\Delta$ that are true in $m$, i.e. $|\Delta|_{m}=\sum_{\beta\in\Delta}\m{\beta}$. Dividing models into $\pams{\Delta}$ and the others, we have 
\begin{eqnarray*}
p(\alpha|\Delta)&=&\lim_{\mu\rightarrow 1}\frac{\sum_{m}p(\alpha|m)p(m)p(\Delta|m)}{\sum_{m}p(m)p(\Delta|m)}\\
&=&\lim_{\mu\rightarrow 1}\frac{\sum_{\hat{m}\in\pams{\Delta}}p(\alpha|\hat{m})p(\hat{m})p(\Delta|\hat{m})+\sum_{m\notin\pams{\Delta}}p(\alpha|m)p(m)p(\Delta|m)}{\sum_{\hat{m}\in\pams{\Delta}}p(\hat{m})p(\Delta|\hat{m})+\sum_{m\notin\pams{\Delta}}p(m)p(\Delta|m)}.
\end{eqnarray*}
Now, $p(\Delta|m)$ can be developed as follows, for all $m$.
\begin{eqnarray*}
p(\Delta|m)&=&\prod_{\beta\in\Delta}p(\beta|m)=\prod_{\beta\in\Delta}\mu^{\m{\beta}_{m}}(1-\mu)^{1-\m{\beta}_{m}}\\
&=&\mu^{\sum_{\beta\in\Delta}\m{\beta}_{m}}(1-\mu)^{\sum_{\beta\in\Delta}(1-\m{\beta}_{m})}=\mu^{|\Delta|_{m}}(1-\mu)^{|\Delta|-|\Delta|_{m}}
\end{eqnarray*}
Therefore, $p(\alpha|\Delta)=\lim_{\mu\rightarrow 1}\frac{W+X}{Y+Z}$ where
\begin{eqnarray*}
W&=&\textstyle{\sum_{\hat{m}\in\pams{\Delta}}p(\alpha|\hat{m})p(\hat{m})\mu^{|\Delta|_{\hat{m}}}(1-\mu)^{|\Delta|-|\Delta|_{\hat{m}}}}\\
X&=&\textstyle{\sum_{m\notin\pams{\Delta}}p(\alpha|m)p(m)\mu^{|\Delta|_{m}}(1-\mu)^{|\Delta|-|\Delta|_{m}}}\\
Y&=&\textstyle{\sum_{\hat{m}\in\pams{\Delta}}p(\hat{m})\mu^{|\Delta|_{\hat{m}}}(1-\mu)^{|\Delta|-|\Delta|_{\hat{m}}}}\\
Z&=&\textstyle{\sum_{m\notin\pams{\Delta}}p(m)\mu^{|\Delta|_{m}}(1-\mu)^{|\Delta|-|\Delta|_{m}}}.
\end{eqnarray*}
Now, if $m\notin\pams{\Delta}$ then $m$ is impossible or $m$ is a possible model of a subset of $\Delta$ that is not a cardinality-maximal possible subset of $\Delta$. Therefore, $p(m)=0$ or there is $\hat{m}\in\pams{\Delta}$ such that $|\Delta|_{m}<|\Delta|_{\hat{m}}$. $|\Delta|_{\hat{m}_{1}}=|\Delta|_{\hat{m}_{2}}$ by definition, for all $\hat{m}_{1},\hat{m}_{2}\in\pams{\Delta}$. The fraction thus can be simplified by dividing the denominator and numerator by $(1-\mu)^{|\Delta|-|\Delta|_{\hat{m}}}$. We thus have $p(\alpha|\Delta)=\lim_{\mu\rightarrow 1}\frac{W'+X'}{Y'+Z'}$ where
\begin{eqnarray*}
W'&=&\textstyle{\sum_{\hat{m}\in\pams{\Delta}}p(\alpha|\hat{m})p(\hat{m})\mu^{|\Delta|_{\hat{m}}}}\\
X'&=&\textstyle{\sum_{m\notin\pams{\Delta}}p(\alpha|m)p(m)\mu^{|\Delta|_{m}}(1-\mu)^{|\Delta|_{\hat{m}}-|\Delta|_{m}}}\\
Y'&=&\textstyle{\sum_{\hat{m}\in\pams{\Delta}}p(\hat{m})\mu^{|\Delta|_{\hat{m}}}}\\
Z'&=&\textstyle{\sum_{m\notin\pams{\Delta}}p(m)\mu^{|\Delta|_{m}}(1-\mu)^{|\Delta|_{\hat{m}}-|\Delta|_{m}}}.
\end{eqnarray*}
Applying the limit operation, we can cancel out $X'$ and $Z'$ and have
\begin{align*}
p(\alpha|\Delta)=\frac{\sum_{\hat{m}\in\pams{\Delta}}p(\alpha|\hat{m})p(\hat{m})}{\sum_{\hat{m}\in\pams{\Delta}}p(\hat{m})}=\frac{\sum_{\hat{m}\in\pams{\Delta}}1^{\m{\alpha}_{\hat{m}}}0^{1-\m{\alpha}_{\hat{m}}}p(\hat{m})}{\sum_{\hat{m}\in\pams{\Delta}}p(\hat{m})}.
\end{align*}
Since $1^{\ms{\alpha}_{\hat{m}}}0^{1-\ms{\alpha}_{\hat{m}}}=1^{1}0^{0}=1$ if $\hat{m}\in\ms{\alpha}$ and $1^{\ms{\alpha}_{\hat{m}}}0^{1-\ms{\alpha}_{\hat{m}}}=1^{0}0^{1}=0$ if $\hat{m}\notin\ms{\alpha}$, we have
\begin{eqnarray}\label{proof:4}
p(\alpha|\Delta)=\frac{\sum_{\hat{m}\in\pams{\Delta}\cap\m{\alpha}}p(\hat{m})}{\sum_{\hat{m}\in\pams{\Delta}}p(\hat{m})}=\frac{\sum_{\hat{m}\in\pams{\Delta}\cap\pms{\alpha}}p(\hat{m})}{\sum_{\hat{m}\in\pams{\Delta}}p(\hat{m})}.
\end{eqnarray}
\qed
\end{proof}
We can now relate probabilistic reasoning on the generative logic model to the consequence relation, i.e., $\ent$, with maximal possible sets.
\begin{corollary}
Let $\{p(M),p(\Gamma|M,\mu)\}$ be a generative logic model such that $\mu\to 1$. For all $\alpha\in\Gamma$ and $\Delta\subseteq\Gamma$ such that $\pams{\Delta}\neq\emptyset$, $p(\alpha|\Delta)=1$ iff $S\ent\alpha$, for all cardinality-maximal possible subsets $S$ of $\Delta$.
\end{corollary}
\begin{proof}
From Equation (\ref{proof:4}), $p(\alpha|\Delta)=1$ iff $\pams{\Delta}\subseteq\pms{\alpha}$. Since $\pams{\Delta}=\bigcup_{S\in MPS(\Delta)}\pms{S}$, $p(\alpha|\Delta)=1$ iff $\bigcup_{S\in MPS(\Delta)}\pms{S}\subseteq\pms{\alpha}$.
\end{proof}
The following theorem holds when $\pams{\Delta}=\emptyset$.
\begin{theorem}\label{thrm:cfr2}
Let $\{p(M),p(\Gamma|M,\mu)\}$ be a generative logic model such that $\mu\to 1$. For all $\alpha\in\Gamma$ and $\Delta\subseteq\Gamma$ such that $\pams{\Delta}=\emptyset$, $p(\alpha|\Delta)=p(\alpha)$.
\end{theorem}
\begin{proof}
Omitted but same as Theorem \ref{thrm:pr2}.
\qed
\end{proof}
\section{Conclusions and Future Work}
We presented a simple probabilistic model to unify perceptual reasoning and logical reasoning. The key problem tackled in this paper is the disconnection between reasoning (the process of deriving new knowledge) and learning (the process of obtaining model parameters). Our model captures the process by which data generate perceptual and logical knowledge. We showed that reasoning with our model is consistent with the existing approaches with MLE. We also showed that reasoning with our model can be characterised in terms of the logical consequence relations given by substantial extensions of the paper \cite{Kido:22}.
\par
There is still much work to be done. The current model can only handle static perception with discrete data. Dynamic perception with/without policy and flexible policy selection \cite{Smith:22} are important future work. The future work also includes the neuroscientific validity of our model and the mathematical analysis of the relationship between our model and variational inference used in free energy principle \cite{friston:10} and predictive coding \cite{Hohwy:08}.
%
\bibliographystyle{splncs04}
\bibliography{btx_kido}
\end{document}